\documentclass[runningheads]{llncs}

\usepackage{graphicx}
\usepackage{amsmath} 
\usepackage{amssymb}
\usepackage{physics}
\usepackage[utf8]{inputenc}
\usepackage[english]{babel}
\usepackage{dsfont}

\DeclareMathOperator*{\argmin}{arg\,min}
\DeclareMathOperator{\Sup}{\textup{Sup}}

\title{An Approach to Colour Morphological Supremum Formation using the LogSumExp Approximation}

\titlerunning{An Approach to Colour Morphology with the LogSumExp Approximation}

\author{Marvin Kahra\inst{1}, Michael Breu\ss\inst{1}, Andreas Kleefeld\inst{2,3}, and Martin Welk\inst{4}}
\authorrunning{M. Kahra et al.}
\institute{Institute for Mathematics, Brandenburg University of Technology Cottbus-Senftenberg, 03046 Cottbus, Germany \\
\email{\{marvin.kahra,breuss\}@b-tu.de} \and
Forschungszentrum Jülich GmbH, Jülich Supercomputing Centre, \\
Wilhelm-Johnen-Stra\ss{}e, 52425 Jülich, Germany \\
\email{a.kleefeld@fz-juelich.de} 
\and
University of Applied Sciences Aachen, Faculty of Medical Engineering and Technomathematics, \\
Heinrich-Mu\ss{}mann-Stra\ss{}e 1, 52428 Jülich, Germany \\
\and
UMIT TIROL – Private University for Health Sciences and Health
Technology, Eduard-Wallnöfer-Zentrum 1, 6060 Hall/Tyrol, Austria \\
\email{martin.welk@umit-tirol.at}
}

\begin{document}

\maketitle

\begin{abstract}
    Mathematical morphology is a part of image processing that has proven to be fruitful for numerous applications. Two main operations in mathematical morphology are dilation and erosion. These are based on the construction of a supremum or infimum with respect to an order over the tonal range in a certain section of the image. The tonal ordering can easily be realised in grey-scale morphology, and some morphological methods have been proposed for colour morphology. However, all of these have certain limitations.

    In this paper we present a novel approach to colour morphology extending upon previous work in the field based on the Loewner order. We propose to consider an approximation of the supremum by means of a log-sum exponentiation introduced by Maslov. We apply this to the embedding of an RGB image in a field of symmetric $2\times2$ matrices. In this way we obtain nearly isotropic matrices representing colours and the structural advantage of transitivity. In numerical experiments we highlight some remarkable properties of the proposed approach.

\end{abstract}

\section{Introduction}

Mathematical morphology is a theory used to analyse spatial structures in images. Over the decades it has developed into a very successful field of image processing, see e.g. \cite{Najman-Talbot,Roerdink-2011,Serra-Soille} for an overview. Morphological operators basically consist of two main components. The first of these is the structuring element (SE), which is characterised by its shape, size and position. These in turn can be divided into two types of SEs, flat and non-flat, cf. \cite{Haralick_1}. A flat SE basically defines a neighbourhood of the central pixel where appropriate morphological operations are performed, while a non-flat SE also contains a mask with finite values that are used as additive offsets. The SE is usually implemented as a window sliding over the image. The second main component is used to perform a comparison of values within an SE. Two basic operations in mathematical morphology are dilation and erosion, where a pixel value is set to the maximum and minimum, respectively, of the discrete image function within the SE. Many morphological filtering procedures of practical interest, such as opening, closing or top hats, can be formulated by combining dilation and erosion operations. Since dilation and erosion are dual operations, it is often sufficient to restrict oneself to one of the two when constructing algorithms.

Let us briefly discuss how the operations of dilation and erosion can be applied to colour morphology, as it is the underlying concept for our further considerations. As already mentioned, the most important operation in morphology is to perform a comparison of the tonal values, or in our case the colour values, within the SE. 
For the simpler application areas such as grey value morphology, one can act directly on complete lattices in order to obtain a total order of the colour values, cf. \cite{Blusseau,Serra-Soille}. In the case of colour morphology, this is no longer the case, as there is no total order of the colour values. For this reason, corresponding half-orders and different basic structures are used, cf. \cite{BAR76}. The first approach that could be used for this would be to regard each colour channel of an image as an independent image and to perform grey value morphology on each of them. The other approach, which is more popular, uses a vector space structure where each colour is considered as a vector in an underlying colour space. This can take the form of \cite{Angulo}, \cite{APT08} or \cite{LEZ07}. We will take the latter approach, but use symmetric matrices instead of vectors. We will order the elements by means of a half-order, namely the Loewner order (see \cite{Loewner}).

To calculate the basic operations of colour morphology, the supremum or infimum is determined instead of the maximum or minimum. There are also different approaches how to choose the supremum of symmetric matrices. To give some examples, we mention here the nuclear norm, the Frobenius norm and the spectral norm. For a comparison of these norms we refer to the work \cite{WelkQuantile} by Welk, Kleefeld and Breu\ss. 

In this paper, however, we want to consider a different approach, namely the approximation of the supremum by means of a log-sum exponentiation. The reason why we follow this approach is that it is an approximation which already gave promising results in the work by Kahra, Sridhar and Breu{\ss} (see \cite{KSB}) for grey scale images in connection with a fast Fourier transform. This approach was also used for colour images in \cite{Srid}, but only in the above described sense of a channel-wise structuring of colour images. Another connection worth mentioning is with the paper of Burgeth, Welk, Feddern and Weickert (see \cite{morph_op_mat_im}), where root and power were used for symmetric positive semidefinite matrices instead of logarithm and exponential function. However, our approach here has the advantage of being transitive.

This paper will be the first step for transferring the LogSumExp approach to colour morphology with tonal vectors/matrices. The goal is to present a first introduction to the characterisation of this approach for tonal value matrices and a comparison with some of the other already existing approaches. 

\section{General Definitions}

\subsection{Colour Morphology}

In order to make this paper self-contained, we briefly recall some basics of mathematical morphology in general and colour morphology in particular.

For this purpose, we first consider a two-dimensional, discrete image domain $\Omega \subseteq \mathbb{Z}^2$ and a single-channel grey-scale image, which is explained by a function $f: \Omega \rightarrow [0,255]$. In the case of non-flat morphology, the SE can be represented as a function $b: \mathbb{Z}^2 \rightarrow \mathbb{R} \cup \{ - \infty\}$ with 
\begin{align*}
    b(\boldsymbol x) := 
    \begin{cases}
        \beta(\boldsymbol x), \quad & \boldsymbol x \in B_0, \\
		- \infty, & \text{otherwise},
    \end{cases}
    \quad B_0 \subset \mathbb{Z}^2,
\end{align*}
where $B_0$ is a set centred at the origin. In the case of a flat filter (flat morphology), it is simply the special case $\beta(\boldsymbol x) = 0$. The most elementary operations of mathematical morphology are dilation 
\begin{align*}
    (f\oplus b)(\boldsymbol x) := \max_{\boldsymbol u\in \mathbb{Z}^2}  {\{f(\boldsymbol x- \boldsymbol u)\: +\: b(\boldsymbol u) \}}
\end{align*}
and erosion
\begin{align*}
    (f\ominus b)(\boldsymbol x) := \min_{\boldsymbol u\in \mathbb{Z}^2}  {\{f(\boldsymbol x + \boldsymbol u)\: -\: b(\boldsymbol u) \}}.
\end{align*}
With these two operations, many other operations can be defined that are of great interest in practice, for example opening $f \circ b = (f \ominus b) \oplus b$ and closing $f \bullet b = (f \oplus b) \ominus b$.

We now come to our real area of interest, namely colour morphology. This is similar to the greyscale morphology already shown, with the difference that we no longer have just one channel, but three. There are many useful formats for expressing this, see \cite{SHA03}. A classic approach in this sense is the channel-by-channel processing of an image with the RGB colour model, see e.g. \cite{Srid} for a recent example of channel-wise scheme implementation. However, instead of RGB vectors, we will use symmetric, positive semi-definite $2 \times 2$ matrices. 
For this we assume that the colour values are already normalised to the interval $[0,1]$. First, we transfer this vector into the HCL colour space by means of $M = \max\{r,g,b\}$, $m = \min\{r,g,b\}$, $c = M - m$, $\ell = \frac{1}{2} (M + m)$ and 
\begin{align*}
    h = 
    \begin{cases}
        \frac{g - b}{6c} \mod 1, \quad & M = r, \\
        \frac{b - r}{6c} + \frac{1}{3} \mod 1, \quad & M = g, \\
        \frac{r - g}{6c} + \frac{2}{3} \mod 1, \quad & M = b.
    \end{cases}
\end{align*}
Then we replace the luminance $\ell$ with $\tilde{\ell} = 2\ell -1$ and consider the quantities $c$, $2\pi h$ and $\tilde{\ell}$ as radial, angular and axial coordinates of a cylindrical coordinate system, respectively. This represents a bijection onto the bi-cone, which in turn is interpreted with Cartesian coordinates by $x = c \cdot \cos(2\pi h)$, $y = c \cdot \sin(2\pi h)$ and $z = \tilde{\ell}$. Finally, we map these Cartesian coordinates onto a symmetric matrix as follows:
\begin{align*}
    \boldsymbol A := \frac{\sqrt{2}}{2} 
    \begin{pmatrix}
        z - y & x \\
        x & z + y
    \end{pmatrix},
\end{align*}
where the complete transformation process is a bijective mapping, see \cite{Loewner}.

\subsection{The Loewner Order and the Decision of a Minimiser}

Since there is no classical ordering of the elements in $\mathbb{R}^3$ or $\mathbb{R}^{2 \times 2}$, we need to ask when an element is larger or smaller than another element. To answer this question, we need to use a weaker definition of an ordering relation, i.e. a half-order. For this reason, we resort to the promising Loewner order and the colour morphological processing based on this, which was already presented by Burgeth and Kleefeld (see \cite{BurgethKleefeld}). 

For two symmetric matrices $\boldsymbol A$ and $\boldsymbol B$, the Loewner order $\leq_L$ or $<_L$ is defined as 
\begin{align*}
    \boldsymbol A <_L \boldsymbol B ~ (\boldsymbol A \leq_L \boldsymbol B) \Longleftrightarrow \boldsymbol B- \boldsymbol A \text{ is positive (semi-)definite.}
\end{align*}
The problem with this semi-order is that it is not a lattice order \cite{BorweinLewis} and therefore it is not possible to find a unique matrix maximum or supremum. To get around this problem, one needs another property to select a uniquely determined matrix maximum from the convex set of symmetric matrices $\mathcal{U}(\mathcal{X})$ that are upper bounds in the Loewner sense for the multi-set $\mathcal{X} = \{ \boldsymbol X_1, \boldsymbol X_2, \dots, \boldsymbol X_n \}$ of given data matrices of the set $\text{Sym}(2)$ of symmetric real $2 \times 2$ matrices with
\begin{align*}
    \mathcal{U}(\mathcal{X}) := \left\{ \boldsymbol Y \in \text{Sym}(2) : \boldsymbol X \leq_L \boldsymbol Y \quad \forall \boldsymbol X \in \mathcal{X} \right\},
\end{align*}
see \cite{WelkQuantile}. For representation of this property, we use the function $\varphi : \mathcal{U}(\mathcal{X}) \rightarrow \mathbb{R}$, which is convex and Loewner-monotone, i.e.
\begin{align*}
    \varphi(\boldsymbol A) \leq \varphi(\boldsymbol B) \quad \Longleftrightarrow \quad \boldsymbol A \leq_L \boldsymbol B. 
\end{align*}
Furthermore, $\varphi$ should have a unique minimiser in $\mathcal{U}(\mathcal{X})$, then we can define the $\varphi$-supremum of $\mathcal{X}$ as said minimiser:
\begin{align*}
    \Sup_{\varphi}(\mathcal{X}) := \argmin_{\boldsymbol Y \in \mathcal{U}(\mathcal{X})} \varphi(\boldsymbol Y).
\end{align*}
The matrix supremum introduced in the paper \cite{morph_op_mat_im} is based on the calculation of the trace and is therefore also called trace-supremum. That is, one has $\varphi(Y) = \trace Y$ and we get as supremum:
\begin{align*}
    \Sup_{\trace}(\mathcal{X}) := \argmin_{\boldsymbol Y \in \mathcal{U}(\mathcal{X})} \trace \boldsymbol Y.
\end{align*}
Based on the corresponding norms, the Frobenius supremum 
\begin{align*}
    \Sup_{2}(\mathcal{X}) := \argmin_{\boldsymbol Y \in \mathcal{U}(\mathcal{X})} \sum_{\boldsymbol X \in \mathcal{X}} \norm{\boldsymbol Y - \boldsymbol X}_2
\end{align*}
and the spectral supremum 
\begin{align*}
    \Sup_{\infty}(\mathcal{X}) := \argmin_{\boldsymbol Y \in \mathcal{U}(\mathcal{X})} \sum_{\boldsymbol X \in \mathcal{X}} \abs{\lambda_1(\boldsymbol Y - \boldsymbol X)},
\end{align*}
where $\lambda_1(\boldsymbol A)$ denotes the largest eigenvalue of $\boldsymbol A$, were derived in \cite{WelkQuantile}. At this point it should be noted that, in the case of positive semidefinite matrices, all three norms are Schatten norms $\norm{\,\cdot\,}_p$ for $p \in \{ 1, 2, \infty \}$.

\section{Mathematical Background of the Log-Exp-Supremum}

In this section, we will give a brief introduction to the construction of a characterisation for the so-called log-exp-supremum and summarise some of its properties, but will refer to a future paper for the proofs. 

Let the multi-set $\mathcal{X} = \{\boldsymbol X_1,\dots , \boldsymbol X_n\}$, $n \in \mathbb{N}$, of symmetric real $2 \times 2$ matrices represent our colour values in the considered neighbourhood. Then, we define the \textbf{log-exp-supremum (LES)} as 
\begin{align}
    \boldsymbol S := \Sup_{\text{LE}}(\mathcal{X}) := \lim_{m \rightarrow \infty} \left( \frac{1}{m} \log \sum_{i=1}^n \exp(m \boldsymbol X_i) \right).
    \label{LES}
\end{align}
For the calculation of the LES, we will use a characterisation of the form
\begin{align}
    \boldsymbol{S} = \lambda_1 \boldsymbol{u}_1 \boldsymbol{u}_1^\textup T + \mu \boldsymbol{v}_1 \boldsymbol{v}_1^\textup T,
    \label{LES_characterisation}
\end{align}
where $\lambda_1$ is the largest eigenvalue of the matrices of $\mathcal{X}$ with the corresponding normalised eigenvector $\boldsymbol{u}_1$, $\boldsymbol{v}_1$ is the normalised eigenvector perpendicular to $\boldsymbol{u}_1$ and $\mu$ is the next largest eigenvalue of $\mathcal{X}$ with an eigenvector not aligned with $\boldsymbol{u}_1$ if $\lambda_1$ is unique, otherwise $\mu = \lambda_1$ holds. At this point it should be noted that $\mu$ does not have to be the corresponding eigenvalue to the eigenvector $\boldsymbol{v}_1$. The characterisation is based on the spectral decomposition of the input $\boldsymbol{X}_i$:
\begin{align}
    \begin{split}
        &\boldsymbol X_i = \lambda_i \boldsymbol u_i \boldsymbol u_i^\textup{T} + \mu_i \boldsymbol v_i \boldsymbol v_i^\textup{T}, \quad \lambda_i \geq \mu_i, \quad \langle \boldsymbol u_i, \boldsymbol v_i \rangle = 0, \quad \abs{\boldsymbol u_i} = 1 = \abs{\boldsymbol v_i}, \\
        &\boldsymbol u_i = (c_i,s_i)^\textup{T} , \quad \boldsymbol v_i = (-s_i,c_i)^\textup{T}, \quad  c_i = \cos(\varphi_i), \quad s_i = \sin(\varphi_i),\\
        &\varphi_i \in \left[- \frac{\pi}{2}, \frac{\pi}{2}\right], \quad i = 1,\dots,n,
    \end{split}
    \label{spectral_decomposition}
\end{align}
where $\lambda_i,\mu_i \in \mathbb{R}$ are the eigenvalues of $\boldsymbol X_i$ and $\boldsymbol u_i \perp \boldsymbol v_i$ are the unit eigenvectors to the corresponding eigenvalues. Additionally, we used for the equation \eqref{LES_characterisation} that without loss of generality the matrix $\boldsymbol{X}_1$ has (one of) the largest eigenvalues of $\mathcal{X}$. The proof exploits the spectral decomposition \eqref{spectral_decomposition} and certain properties of the Rayleigh product and symmetric $2 \times 2$ matrices. However, we will not go into the proof of the LES characterisation \eqref{LES_characterisation} here, as it would strain the scope of this paper too much. Instead, we intend to do this in a separate paper and refer to the results of our experiments for now.

We now turn our attention to some of the properties of \eqref{LES}. First of all, we need to clarify whether it is actually a supremum in the Loewner sense. In fact, it is easy to show that it is an upper bound:

\begin{lemma}
    The LES $\boldsymbol S$ according to \eqref{LES} is an upper bound in the Loewner sense for the given matrices $\mathcal{X}$.
\end{lemma}
\begin{proof}
    Apparently applies 
    \begin{align*}
        \sum_{i=1}^n \exp(m \boldsymbol X_i) \geq_L \exp(m \boldsymbol X_j) \quad \forall j \in \{ 1,\dots,n \},  
    \end{align*}
    which in combination with the fact that the logarithm is an operator-monotone function (see \cite{lowner_monotone}), i.e. $\boldsymbol A \leq_L \boldsymbol B \Longrightarrow \log \boldsymbol A \leq_L \log \boldsymbol B$ for symmetric positive definite matrices $\boldsymbol A,\boldsymbol B$, results in $\boldsymbol S \geq_L \boldsymbol X_j$ for all $j \in \{1,\dots,n\}$. \hfill $\square$
\end{proof}

We already mentioned in the previous section that it is necessary to have another property to select a uniquely determined maximum from the convex set $\mathcal{U}(\mathcal{X})$. This should be represented by a function $\varphi$, which should have a unique minimiser in $\mathcal{U}(\mathcal{X})$. Unfortunately, there is no such total ordering function $\varphi$ on $\mathcal{U}(\mathcal{X})$ for which $\boldsymbol{S}$ is the unique minimiser. However, by defining the \textbf{$\boldsymbol{p}$-power upper bound cone} as
\begin{align}
    \mathcal{U}_p(\mathcal{X}) := \left( \mathcal{U}(\mathcal{X}^p) \right)^\frac{1}{p} = \big\{ \boldsymbol{Y} \in \textup{Sym}(2) : \boldsymbol{Y}^p \in \mathcal{U}(\mathcal{X}^p) \big\}
    \label{p-power UB cone}
\end{align}
for the element-wise application of the $p$-th power to the multi-set $\mathcal{X}$ in form of $\mathcal{X}^p := \{ \boldsymbol{X}^p :  \boldsymbol{X} \in \mathcal{X} \}$, we can slightly modify the problem to obtain a unique minimiser. To be exact, we denote the intersection of all $p$-power upper bound cones as the \textbf{super-upper bound cone}
\begin{align}
    \mathcal{U}_*(\mathcal{X}) := \bigcap_{p > 0} \mathcal{U}_p(\mathcal{X})
    \label{super-UB cone}
\end{align}
and use this cone to find the minimiser instead of $\mathcal{U}(\mathcal{X})$. In fact, following the reasoning in \cite{morph_op_mat_im}, one sees that $\boldsymbol{S} \in \mathcal{U}_p(\mathcal{X})$ holds for any $p > 0$, and therefore $\boldsymbol{S} \in \mathcal{U}_*(\mathcal{X})$. By declaring the function $\boldsymbol{\varphi}: \textup{Sym}(2) \rightarrow \mathbb R^2$ with $\boldsymbol{\varphi}(\boldsymbol{A}) = (\lambda_1,\lambda_2)$, where $\lambda_1,\lambda_2$ are the eigenvalues of $\boldsymbol{A} \in \textup{Sym}(2)$, and endowing the $\mathbb R^2$ with the lexicographic order
\begin{align}
    (a,b) \prec (a',b') :\Longleftrightarrow (a < a' \vee (a = a' \wedge b \leq b'))
    \label{lexicographic order}
\end{align}
one can show that the function $\boldsymbol{\varphi}$ is Loewner-monotone, convex and $\boldsymbol{S}$ is its unique minimiser in $\mathcal{U}_*(\mathcal{X})$. 
In addition, one can show the following lemma for transitivity, which we will not do here, but refer to our experiments.
\begin{lemma}
    The LES \eqref{LES} is transitive, i.e. for multi-sets $\mathcal{X}$ and $\mathcal{Y}$ of symmetric $2 \times 2$ matrices one has 
    \begin{align}
        \Sup_{\textup{LE}}(\mathcal{X} \cup \mathcal{Y}) = \Sup_{\textup{LE}}\big(\{\Sup_{\textup{LE}}(\mathcal{X}), \Sup_{\textup{LE}}(\mathcal{Y})\}\big).
        \label{Transitivity}
    \end{align}
\end{lemma}

\section{Experiments}

We divide this section into three parts with four experiments. First, we will show that our new method yields correct results in the Loewner sense, as did the previous Loewner methods. Then we will show how the approximations of different morphological operations affect it. Our main focus will be on the almost isotropic character and especially the transitivity compared to previous approaches. 
For this, we will apply some elementary flat morphological operations to synthetic and natural colour images of different sizes. For the experiments we are using various structuring elements: a $3 \times 3$ square, a $5 \times 5$ square and a $9 \times 9$ square
, all centred on the middle pixel. 

\subsection{Correctness of the Loewner Order}

For this we consider a simple $30 \times 30$ image divided into blue (RGB = (0,0,1)) and green (RGB = (0,1,0)) and perform a dilation and an erosion on it with the $9 \times 9$ SE, see Figure \ref{bgimage}. To better distinguish the images, we have placed a black frame around them. The white border in between does not belong to the image. 

\begin{figure}
\centering
\minipage{0.2\linewidth}
 \fbox{\includegraphics[width=\linewidth]{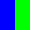}}
\endminipage
\hfill
\minipage{0.2\linewidth}
 \fbox{\includegraphics[width=\linewidth]{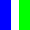}}
\endminipage
\hfill
\minipage{0.2\linewidth}
 \fbox{\includegraphics[width=\linewidth]{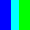}}
\endminipage
\hfill
\minipage{0.2\linewidth}
 \fbox{\includegraphics[width=\linewidth]{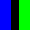}}
\endminipage
\caption{\label{bgimage} Comparison of different dilation and erosion methods of an $30 \times 30$ image with a $9 \times 9$ SE.
{\bf From left to right:} Original image, dilation with the new method, channel-wise dilation and erosion (new method or channel-wise).
}
\end{figure}

Since both colours are on opposite sides of the $x$-$y$-plane of the bi-cone with respect to the $y$-coordinate, the maximum and minimum should be white and black, respectively. The dilation with our approach correctly shows white and the erosion black, as was to be expected from earlier Loewner approaches, compare the work \cite{BurgethKleefeld} of Burgeth and Kleefeld, since all colour channels are considered simultaneously. In contrast, at least channel-wise dilation, i.e. separate dilation for each colour channel, does not produce the desired white according to the Loewner order, but a mixture of the blue and green tones. 

\subsection{Application of different Morphological Operations}

We will use the $64 \times 64$ pepper image and apply dilation, erosion, closing and opening with the $3 \times 3$ SE according to our approach, see Figure \ref{char_operations}. We immediately notice that most of the colours fade. As a result of log-sum exponentiation, we basically only compare the leading eigenvalues with different eigenvectors, and as $m$ increases, the leading terms become more dominant, which results in higher eigenvalues and thus brighter colours. This is especially true if the two largest eigenvalues in the neighbourhood under consideration are already close to each other, which often happens with smaller SEs and natural images. If the two largest eigenvalues $\lambda_1$ and $\lambda_2 = \lambda_1 + \varepsilon$ for a sufficiently small $\varepsilon > 0$ have different eigenvectors, we obtain with our usual notation
\begin{align*}
    \boldsymbol{S} = \lambda_2 \boldsymbol u_2 \boldsymbol u_2^\textup T + \lambda_1 \boldsymbol v_2 \boldsymbol v_2^\textup T = (\lambda_1 +  \varepsilon) \boldsymbol u_2 \boldsymbol u_2^\textup T + \lambda_1 \boldsymbol v_2 \boldsymbol v_2^\textup T = \lambda_1 \boldsymbol{I} + \varepsilon \boldsymbol u_2 \boldsymbol u_2^\textup T,
\end{align*}
where we set $\boldsymbol u_2 = (1,0)^\textup T$ and $\boldsymbol v_2 = (0,1)^\textup T$. For example, the exact dilation in Figure \ref{char_operations} has an average difference between the largest and second largest eigenvalue of 0.0268 per pixel, while the largest and smallest eigenvalues of the image are 0.5574 and $-0.7071$. To illustrate this further, we have included in Figure \ref{char_operations} an approximate dilation with a scaling factor of $m = 69$, chosen because this is the largest factor that Matlab can compute without error. The clear difference to the converged dilation is possibly due to two reasons. Firstly, the factor is far from the limit value, but this cannot be changed, at least with the current implementation in Matlab, for the reason mentioned above. Secondly, it may be due to the fact that if the eigenvector directions are close together, which is not unusual in natural images where there are no fast or strong colour transitions, the convergence against the converged matrix should be quite slow.

\begin{figure}
\centering
\includegraphics[width=0.29\linewidth]{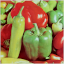} \hspace{0.6em}
\includegraphics[width=0.29\linewidth]{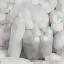} \hspace{0.6em}
\includegraphics[width=0.29\linewidth]{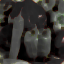}
\\[0.6em]
\includegraphics[width=0.29\linewidth]{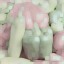} \hspace{0.6em}
\includegraphics[width=0.29\linewidth]{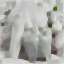} \hspace{0.6em}
\includegraphics[width=0.29\linewidth]{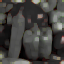}
\caption{\label{char_operations} Illustration of exact filtering results obtained with the new LES method for different operations of a $64 \times 64$ pepper image with a $3 \times 3$ SE. 
{\bf Top:} {\bf From left to right:} Original image, exact dilation and exact erosion.
{\bf Bottom:} {\bf From left to right:} Approximate dilation ($m = 69$), exact closing and exact opening.
}
\end{figure}


\subsection{Transitivity}

For this experiment we again use the $64 \times 64$ pepper image and apply a dilation with the $5 \times 5$ mask to it and compare the result with applying a dilation with the $3 \times 3$ mask twice for our method and Burgeth's and Kleefeld's trace optimisation method, see Figure 3. By comparison we mean a channel-wise comparison of the two images in the form of a difference from which we take the absolute values. Our method should give the exact same image due to transitivity and Burgeth's and Kleefeld's method \cite{Loewner} should give a slightly different image.

\begin{figure}
\centering
\minipage{0.32\linewidth}
 \includegraphics[width=\linewidth]{images/pepper64.png}
\endminipage
\hfill
\minipage{0.32\linewidth}
 \includegraphics[width=\linewidth]{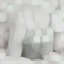}
\endminipage
\hfill
\minipage{0.32\linewidth}
 \includegraphics[width=\linewidth]{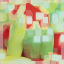}
\endminipage
\hfill
\\[0.6em]
\minipage{0.32\linewidth}
 \includegraphics[width=\linewidth]{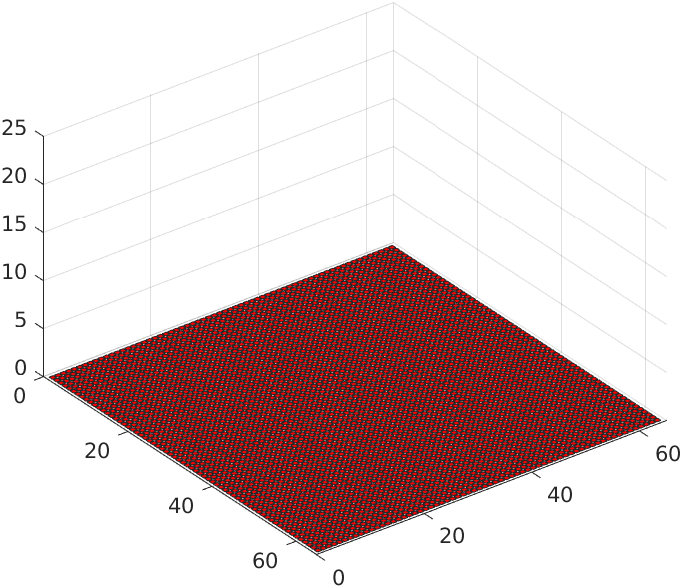}
\endminipage
\hfill
\minipage{0.32\linewidth}
 \includegraphics[width=\linewidth]{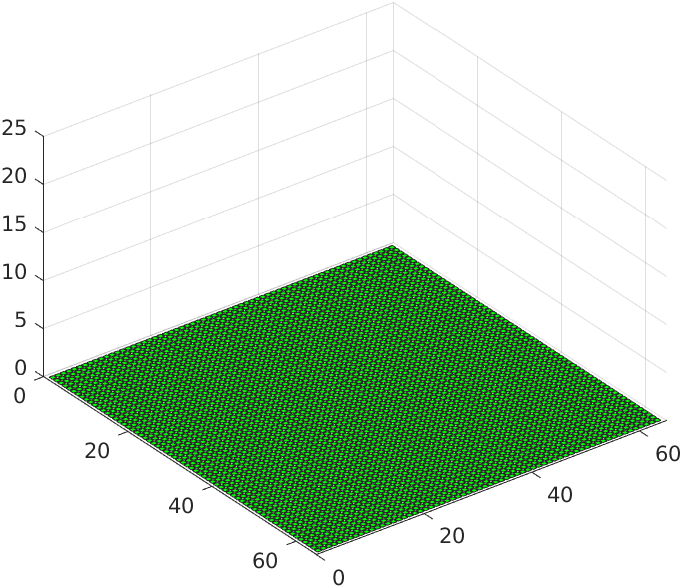}
\endminipage
\hfill
\minipage{0.32\linewidth}
 \includegraphics[width=\linewidth]{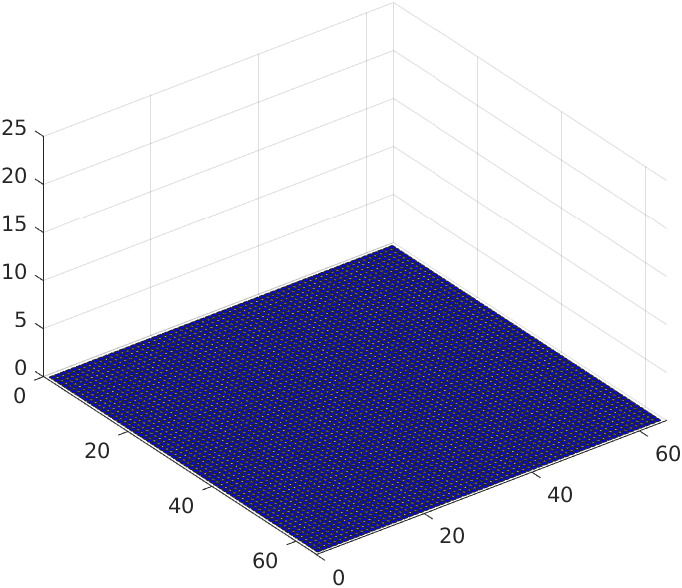}
\endminipage
\hfill
\\[0.6em]
\minipage{0.32\linewidth}
 \includegraphics[width=\linewidth]{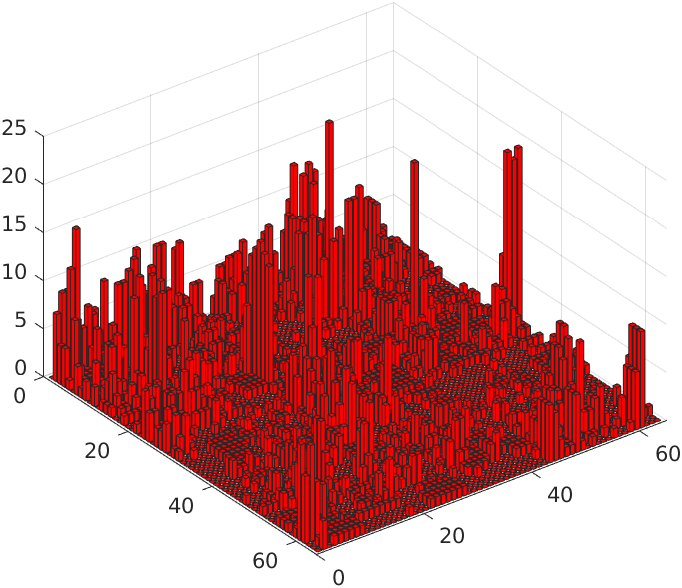}
\endminipage
\hfill
\minipage{0.32\linewidth}
 \includegraphics[width=\linewidth]{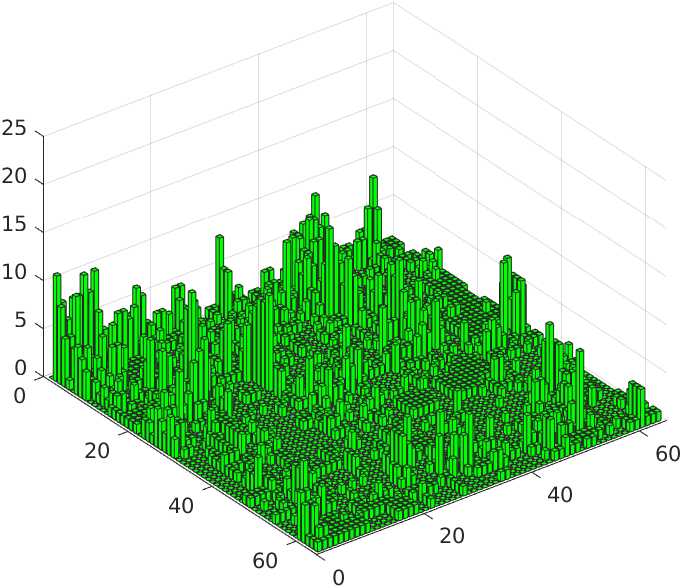}
\endminipage
\hfill
\minipage{0.32\linewidth}
 \includegraphics[width=\linewidth]{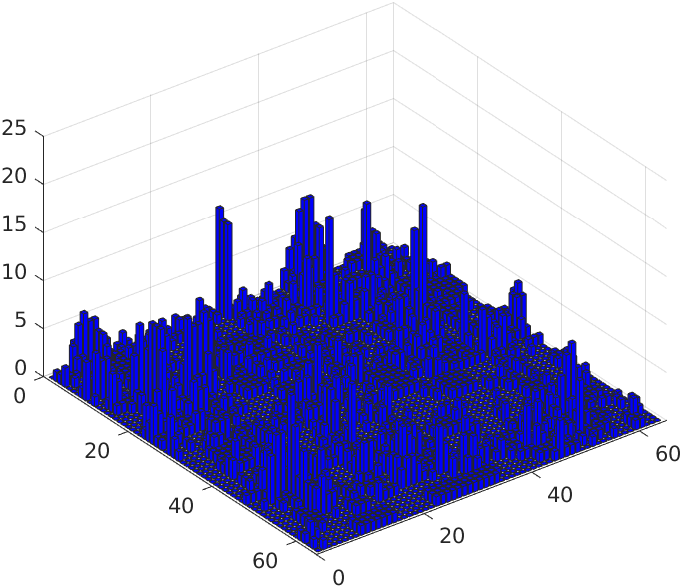}    
\endminipage
\caption{\label{transitivity_5-3} Comparison of differences in the channel-wise deviations between the exact filtering results of the dilation with a $5 \times 5$ SE and applying a dilation with the $3 \times 3$ SE twice using our new method and the method of Burgeth and Kleefeld. 
{\bf Top:} {\bf From left to right:} Original image, dilation with the $5 \times 5$ SE by using our method and Burgeth's and Kleefeld's method.
{\bf Mid:} {\bf From left to right:} Differences in our method with absolute values in the channels red, green and blue.
{\bf Bottom:} {\bf From left to right:} Differences in Burgeth's and Kleefeld's method with absolute values in the channels red, green and blue.
}
\end{figure}

In fact, with our method we do not see any differences in the colour channels. In contrast, Burgeth's and Kleefeld's method shows clear deviations, which occur across the entire image in all channels. In particular, the red colour channel shows strong deviations of up to around 20 tonal values. 

To examine this difference more closely, we look at how the error behaves when we double the number of dilations. This means we compare a dilation with a $9 \times 9$ SE with a fourfold dilation with a $3 \times 3$ SE, see Figure \ref{transitivity_9-3}. Because of the doubling of operations, we would expect the error for the corresponding methods to double as well.

\begin{figure}
\centering
\minipage{0.32\linewidth}
 \includegraphics[width=\linewidth]{images/pepper64.png}
\endminipage
\hfill
\minipage{0.32\linewidth}
 \includegraphics[width=\linewidth]{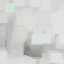}
\endminipage
\hfill
\minipage{0.32\linewidth}
 \includegraphics[width=\linewidth]{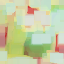}
\endminipage
\hfill
\\[0.6em]
\minipage{0.32\linewidth}
 \includegraphics[width=\linewidth]{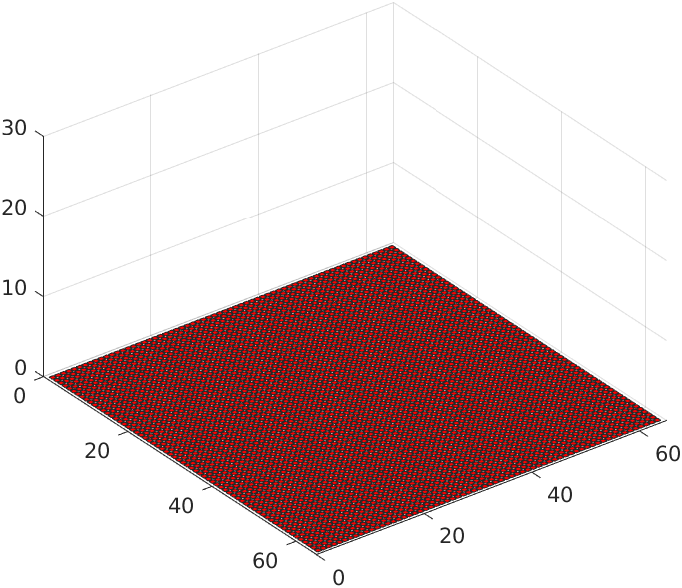}
\endminipage
\hfill
\minipage{0.32\linewidth}
 \includegraphics[width=\linewidth]{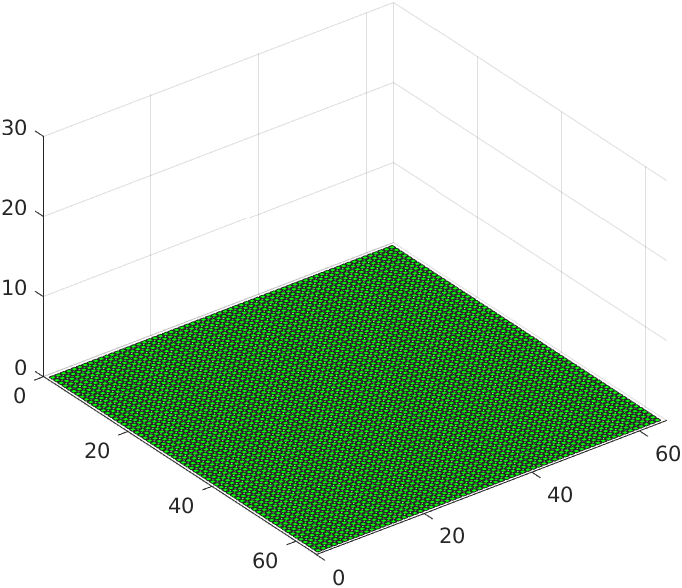}
\endminipage
\hfill
\minipage{0.32\linewidth}
 \includegraphics[width=\linewidth]{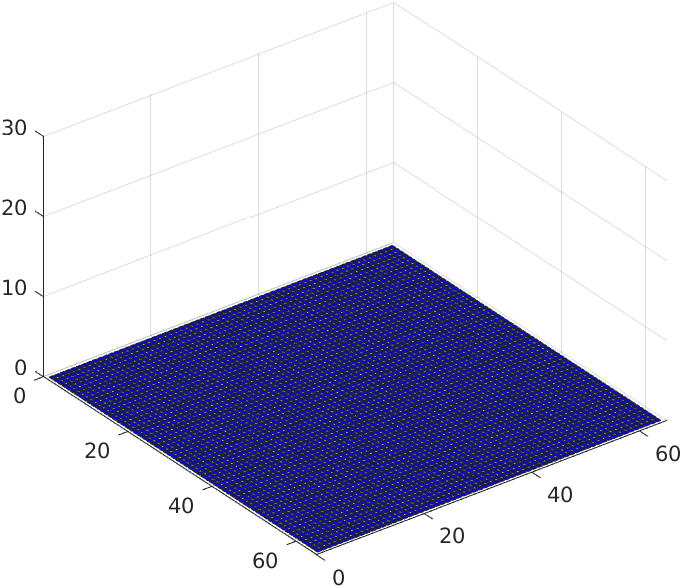}
\endminipage
\hfill
\\[0.6em]
\minipage{0.32\linewidth}
 \includegraphics[width=\linewidth]{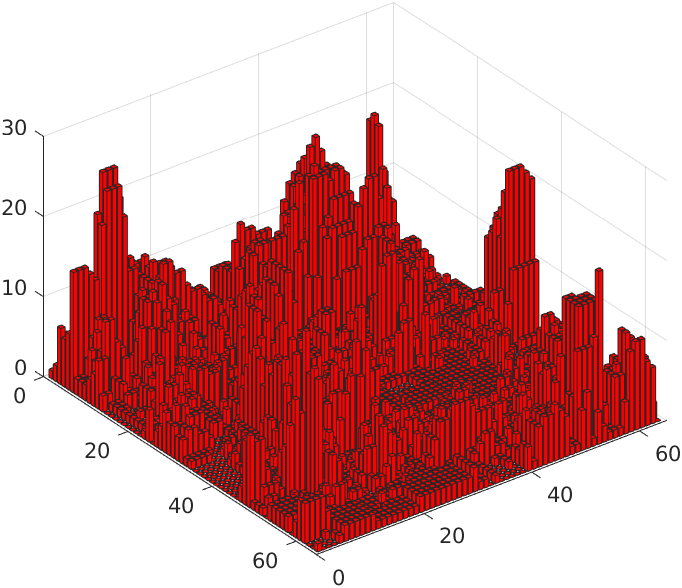}
\endminipage
\hfill
\minipage{0.32\linewidth}
 \includegraphics[width=\linewidth]{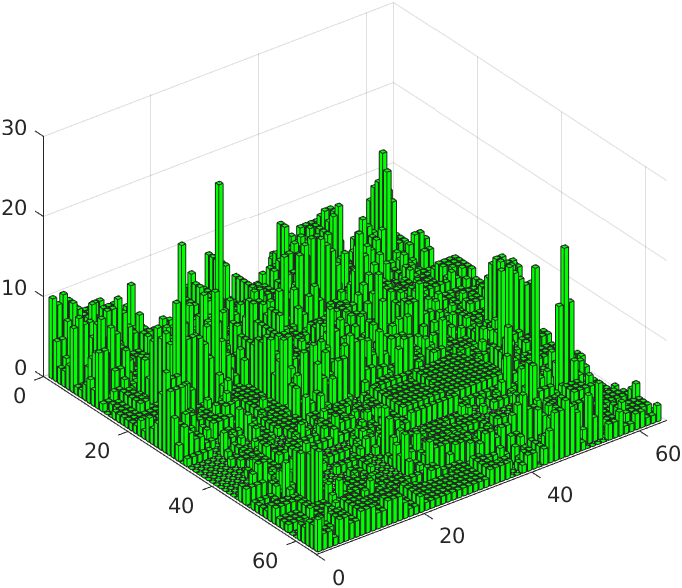}
\endminipage
\hfill
\minipage{0.32\linewidth}
 \includegraphics[width=\linewidth]{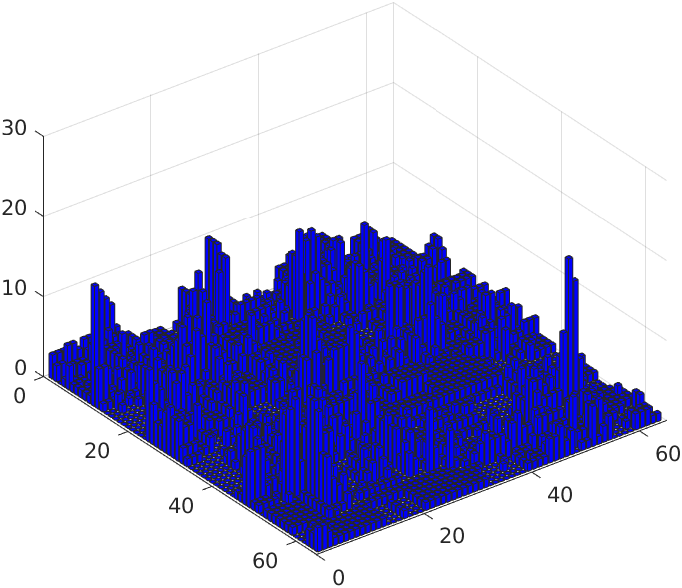}    
\endminipage
\caption{\label{transitivity_9-3} Comparison of differences in the channel-wise deviations between the exact filtering results of the dilation with a $9 \times 9$ SE and fourfold dilation with the $3 \times 3$ SE using our new method and the method of Burgeth and Kleefeld. 
{\bf Top:} {\bf From left to right:} Original image, dilation with the $9 \times 9$ SE by using our method and Burgeth's and Kleefeld's method.
{\bf Mid:} {\bf From left to right:} Differences in our method with absolute values in the channels red, green and blue.
{\bf Bottom:} {\bf From left to right:} Differences in Burgeth's and Kleefeld's method with absolute values in the channels red, green and blue.
}
\end{figure}

\noindent We see a continuation of the previous behaviour in the individual channels, which means that the transitivity of our method is preserved and that errors still occur everywhere with Burgeth's and Kleefeld's method, but particularly in the red colour channel. We see a further increase in the size of the errors, but it is not clear whether the error has actually doubled. 

To measure the error more accurately, we use the Frobenius norm, which we apply to the channel-by-channel difference of the images, and then add this up for the individual channels. In this way, we obtain the value 336.9095 for the first case of a twofold dilation and the value 686.5281 for the second case of a fourfold dilation, which is almost twice the error.

\section{Conclusion and Future Work}

In this work, we have used the colour morphology used by Burgeth and co-authors, see \cite{Loewner,BurgethKleefeld,morph_op_mat_im}, using Loewner order to introduce a new method for determining a ``maximum colour" in the sense of dilation. We applied the log-sum exponentiation to approximate the supremum, which was introduced by Maslov \cite{Maslov}, to the symmetric $2 \times 2$ matrices, which were created using a bijective mapping from the HC$\tilde{\textup L}$ bi-cone. By using the spectral decomposition of these matrices, we have managed to express the result of the approximation using these spectral decompositions, which means that it only depends on the input data. We conducted several experiments on synthetic and natural images to demonstrate Loewner order compliance, how different morphological operations affect our approach, and identified some key factors that influence the approximation. In particular, we have shown that our approach fulfils the property of transitivity, which is important for lattices and orders, in contrast to Burgeth's and Kleefeld's method \cite{BurgethKleefeld} or other methods derived from it, cf. \cite{WelkQuantile}.

In future work, we will take a closer look at the construction briefly introduced here and prove the properties presented. In this regard, we will particularly examine the minimality of the approach in more detail and provide a relaxation for the LES \eqref{LES} so that it even depends continuously on the input data.

\bibliographystyle{splncs04}
\bibliography{MyBibliography}

\end{document}